\documentclass[twoside,11pt]{article}
\usepackage{style}

\ShortHeadings{Thompson Sampling for Gaussian Entropic Risk Bandits}{Ang, Lim, and Chang}
\firstpageno{1}

\begin{document}

\title{Thompson Sampling for Gaussian Entropic Risk Bandits}

\author{\name Ang Ming Liang \email angmingliang@u.nus.edu \\
       \addr Faculty of Science, Computational Biology and Mathematics\\
       National University of Singapore\\
       21 Lower Kent Ridge Rd, Singapore 119077
       \AND
       \name Eloise Y. Y. Lim \email limeloiseyy@u.nus.edu \\
       \addr Faculty of Enginnering, Industrial System Engineering\\
       National University of Singapore\\
       21 Lower Kent Ridge Rd, Singapore 119077
       \AND
       \name Joel Q. L. Chang \email joel.chang@u.nus.edu \\
       \addr Faculty of Science, Mathematics\\
       National University of Singapore\\
       21 Lower Kent Ridge Rd, Singapore 119077
       }

\maketitle

\begin{abstract}
The multi-armed bandit (MAB) problem is a ubiquitous decision-making problem that exemplifies the exploration-exploitation tradeoff. Standard formulations exclude risk in decision making. Risk notably complicates the basic reward-maximising objectives, in part because there is no universally agreed definition of it. In this paper, we consider an entropic risk (ER) measure and explore the performance of a Thompson sampling-based algorithm ERTS under this risk measure by providing regret bounds for ERTS and corresponding instance dependent lower bounds. 
\end{abstract}


\section{Introduction}

The multi-armed bandit (MAB) problem is a classic reinforcement learning problem that analyses sequential decision making, in which the learner has access to partial feedback from her decisions. The problem has been garnering interest in recent years. It informs many critical theoretical questions about the role of exploration vs exploitation in reinforcement learning and applies to both theoretical problems and various real-world applications, such as dynamic pricing, clinical trials, and portfolio optimisation. 
\\\\
In the well-known stochastic MAB setting, a player chooses among $K$ arms, each characterised by an independent reward distribution. During each period, the player plays one arm and observes a random reward from that arm. She then incorporates the information she receives from pulling that arm in choosing the next arm she selects. The player repeats the process for a horizon of $n$ periods. In each period, the player faces a dilemma between exploring other arms' potential value or exploiting the arm that the player believes offers the highest estimated reward. 
\\\\
In the usual setting, the risk of pulling an arm is not being taken into account. However, in many practical settings, such as financial portfolio optimisation, the risk is often the clients' main concern. In this regard, the MAB problem can been tweaked to model such risk-aversion. This paper uses entropic risk measure as the risk measure to minimise due to the simple exponential relationship it has with risk-aversion and utility, and devises a Thompson sampling-based learning algorithm that minimises entropic risk.
\subsection{Related Work}
A variety of analyses on MABs involving risk measures have been carried out. \citet{sani2013riskaversion} considered the mean-variance as their risk measure. Each arm $i$ followed a Gaussian distribution with mean $\mu_i \in [0, 1]$ and variance $\sigma_i^2 \in [0, 1]$. The authors provided an LCB-based algorithm with accompanying regret analyses. \citet{galichet2013exploration} proposed the Multi-Armed Risk-Aware Bandit (\textsc{MaRaB}) algorithm with the goal of minimising the number of pulls of risky arms, using the risk measure CVaR. \citet{Vakili_2016} demonstrated that the instance-dependent and instance-independent regrets in terms of the mean-variance of the reward process over a horizon $n$ are lower bounded by $\Omega(\log n)$ and $\Omega(n^{2/3})$ respectively. \citet{sun2016riskaware} analysed contextual bandits with risk constraints, and developed a meta algorithm which makes use of the online mirror descent algorithm that achieves near-optimal regret with respect to minimising the total cost. \citet{zhu2020thompson} designed the first Thompson sampling algorithm for risk measures, particularly the mean-variance risk measure for Gaussian bandits, and proved near-optimal regret bounds under specific regimes. \citet{chang2021riskconstrained} designed a Thompson sampling algorithm factoring a user's ``risk tolerance'' level, either minimising mean rewards under some ``maximum risk'' criterion, or simply minimising the risk measure. \citet{baudry2020thompson} designed and analysed Thompson sampling-based algorithms $\alpha$-NPTS for bounded rewards and $\alpha$-Multinomial-TS for discrete multinomial distributions.
\\\\
The papers most related to our work is that by \citet{zhu2020thompson} and \citet{chang2021riskconstrained}. \citet{zhu2020thompson} considered arms with the \emph{highest} mean-variance to be optimal, and their definitions and methods can be analogously defined for \emph{minimising} the mean-variance. \citet{chang2021riskconstrained} defined arms with the minimum CVaR as optimal in their ``infeasible instance'', which produced theoretical analogues for ``feasible instances''. This hints that the heavy duty analysis happens in trying to choose arms with the risk measure minimised. Our paper seeks to explore the efficacy of Thompson sampling in the analogous risk-minimising problem setting proposed by  \citet{zhu2020thompson}, but instead considering the \emph{entropic risk} measure. We demonstrate and prove the asymptotic optimality of ERTS, whose asymptotic upper bound matches the theoretical lower bound for consistent algorithms that solve the entropic risk MAB for Gaussian bandits.
\subsection{Contributions}
\begin{itemize}
    \item \textbf{ERTS Algorithm:}
We design ERTS, an algorithm that is similar to the structure of CVaR-TS in \citet{chang2021riskconstrained} but using entropic risk instead of CVaR as the risk measure. This algorithm uses Thompson sampling \citep{thompson1933likelihood} as explored for mean-variance bandits in \citet{zhu2020thompson} and CVaR bandits in \citet{chang2021riskconstrained}.

    \item \textbf{Comprehensive regret bounds:}
We provide theoretical analysis of the ERTS algorithm for Gaussian bandits with bounded variances. We state and prove both upper and lower bounds, showing that ERTS is the first asymptotically optimal algorithm that solves the entropic risk multi-armed bandit problem. Our proof techniques solidify the novel $\xi$-trick in \citet{chang2021riskconstrained}, and affirm future analysis on MABs involving generalised risk measures.

\end{itemize}

This paper is structured as follows. We first introduce the formulation of the entropic risk MAB problem in Section~\ref{sec: er_mab}. In Section~\ref{sec: er_ts}, we present ERTS algorithm. We present our regret bounds and prove that the upper bound we derived is asymptotically optimal in Section~\ref{sec: regret_bounds}. Section~\ref{sec: pf_outline} provides the proof outlines of the regret bounds in Section~\ref{sec: regret_bounds}. We conclude our discussion in Section~\ref{sec: conc} summarizing our work and suggesting avenues for further research. For brevity, we defer detailed proofs of the theorems to the supplementary material. 
\section{Problem formulation}\label{sec: er_mab}
In this section we define the entropic risk MAB problem. For the rest of the paper, denote $[k] = \{1,\dots,k\}$ for any $k \in \NN$ and ${(t)}^+ = \max\{0,t\}$ for $t \in \RR$.
\begin{definition}
\em For any random variable $X$, given a risk parameter $\gamma$, the \emph{entropic risk} \citep{lee2020learning, Howard72} of $X$ is defined by $$\ER_\gamma(X) := \frac{1}{\gamma} \log \EE[\exp(-\gamma X)].$$
\end{definition}
In this paper, we work with Gaussian random variables $X \sim \ccal N(\mu,\sigma^2)$. Direct computations then yield $\ER_\gamma(X) = -\mu + (\gamma/2) \sigma^2$, which is consistent with the computation in \citet[Section 5]{chang2021riskconstrained}. Setting $\gamma \to 0^+$ (resp. $\gamma \to +\infty$) yields the risk-neutral (resp. risk-averse) setting, since $\mu$ (resp. $\sigma^2$) dominates in the former (resp. latter) case.
\\\\
Consider a $K$-armed MAB $\nu = {(\nu(i))}_{i \in [K]}$ played over a horizon of length $n$. Letting $\ER_\gamma(X)$ denote the entropic risk, our objective is to select the least risky arm, that is, the arm with the lowest entropic risk. Thus, we define an arm $i$ to be optimal precisely when $i \in \argmin_{k \in [K]} \ER_\gamma(\nu(k))$. Suppose arm $1$ is optimal (uniquely, without loss of generality). We can then define $\ERgap{i} := \ER_\gamma(i) - \ER_\gamma(1) > 0$ and the \emph{regret} of a policy $\pi$ by $$\ccal R_n(\pi) := \sum_{i \in \sdiff{[K]}{\{1\}}} \EE[T_{i,n}] \ERgap{i},$$ where $T_{i,n}$ denotes the number of times arm $i$ was pulled in the first $n$ rounds. This is a natural definition based on regret decomposition \citep[Chapter 4.5]{lattimore_szepesvari_2020}, and in fact corresponds to the regret decomposition in the case $\gamma \to 0^+$ (i.e. the risk-neutral setting). In the following, we design and analyse ERTS, which aims to minimise $\ccal R_n(\pi)$, and also attain an instance-dependent lower bound, which establishes asymptotic optimality.
\section{The ERTS Algorithm}\label{sec: er_ts}
In this section, we introduce the Entropic Risk Thompson Sampling (ERTS) algorithm for Gaussian bandits with bounded variances, i.e., $\nu \in \ccal E_{\ccal N}^K(\sigma_{\max}^2) := \{\nu=(\nu_1,\dots,\nu_K) : \nu_i \sim \ccal N(\mu_i,\sigma_i^2), \sigma_i^2 \leq \sigma_{\max}^2\ \forall i \in [K]\}$ for some $\sigma_{\max}^2 > 1$. Similar to \citet{zhu2020thompson} and \citet{chang2021riskconstrained}, the algorithm samples from the posteriors of each arm, then chooses the arm according to a multi-criterion procedure.
\\\\
Denote the mean and precision of the Gaussian by $\mu$ and $\psi$ respectively. If $(\mu,\psi)$ follows the distribution $\mathrm{Normal}\text{-}\mathrm{Gamma}(\mu,T,\alpha, \beta)$, then $\psi \sim \mathrm{Gamma}(\alpha,\beta)$, and $\mu | \psi \sim \ccal N(\mu,1/(\psi T))$. Since the conjugate prior for the Gaussian with unknown mean and variance is the Normal-Gamma distribution, we use Algorithm~\ref{alg: update} to update $(\mu,\psi)$ via Bayes' theorem.
\\\\
We present the ERTS algorithm. In each round $t$, for each arm $i$, the player samples the parameters $(\theta_{it},\kappa_{it})$ from the posterior distribution of arm $i$, then chooses arm $j = \argmin_{i \in [k]} \widehat{\ER}_\gamma(i,t)$, where $\widehat{\ER}_\gamma(i,t) := -\theta_{i,t} + \gamma/(2 \kappa_{i,t})$, i.e. least risky arm available.
\begin{algorithm}[ht]
  \caption{$\Update(\hat{\mu}_{i,t-1}, T_{i,t-1}, \alpha_{i,t-1}, \beta_{i,t-1})$}
  \label{alg: update}
\begin{algorithmic}[1]
	\STATE {\bfseries Input:} Prior parameters $(\hat{\mu}_{i,t-1}$, $T_{i,t-1}$, $\alpha_{i,t-1}$, $\beta_{i,t-1})$ and new sample $X_{i,t}$
	\STATE Update the mean: $\hat{\mu}_{i,t} = \frac{T_{i,t-1}}{T_{i,t-1}+1}\hat{\mu}_{i,t-1} + \frac{1}{T_{i,t-1}+1} X_{i,t}$
	\STATE Update the number of samples, the shape parameter, and the rate parameter: $T_{i,t} = T_{i,t-1} + 1$, $\alpha_{i,t} = \alpha_{i,t-1} + \frac{1}{2}$, $\beta_{i,t} = \beta_{i,t-1} + \frac{T_{i,t-1}}{T_{i,t-1}+1} \cdot \frac{{(X_{i,t} - \hat{\mu}_{i,t-1})}^2}{2}$
\end{algorithmic}
\end{algorithm}
\begin{algorithm}[ht]
  \caption{Entropic Risk Thompson Sampling (ERTS)}
  \label{alg: ERTS}
\begin{algorithmic}[1]
  \STATE {\bfseries Input:} Risk parameter $\gamma$, $\hat{\mu}_{i,0} = 0$, $T_{i,0} = 0$, $\alpha_{i,0} = \frac{1}{2}$, $\beta_{i,0} = \frac{1}{2}$
  \FOR{$t=1,2,\ldots, K$}
  \STATE Play arm $t$ and update $\hat{\mu}_{t,t} = X_{t,t}$
  \STATE $\Update(\hat{\mu}_{t,t-1}, T_{t,t-1}, \alpha_{t,t-1}, \beta_{t,t-1})$
  \ENDFOR
  \FOR{$t=K+1,K+2,...$}
  \STATE Sample $\kappa_{i,t}$ from $\mathrm{Gamma}(\alpha_{i,t-1},\beta_{i,t-1})$
  \STATE Sample $\theta_{i,t}$ from $\mathcal{N}(\hat{\mu}_{i,t-1}, 1/T_{i,t-1})$
  \STATE Play arm $j(t) = \argmin_{i \in [K]} \widehat{\ER}_{\gamma}(i,t)$ and observe loss $X_{j(t),t}\sim \nu(j(t))$
  \STATE $\Update(\hat{\mu}_{j(t),t-1}, T_{j(t),t-1}, \alpha_{j(t),t-1}, \beta_{j(t),t-1})$
  \ENDFOR
\end{algorithmic}
\end{algorithm}

\section{Regret Bound for ERTS and Lower Bounds}\label{sec: regret_bounds}
We present our regret bounds in the following theorems. These verify the conjecture made in \citet[Section 5]{chang2021riskconstrained} regarding risk measures of Gaussian bandits of the form $af(\mu) + b g(\sigma^2)$, where $(f(x),g(x),a,b) = (x,x,-1,\gamma/2)$. Furthermore, they establish ERTS as asymptotically optimal in the context of Gaussian entropic risk bandits.
\begin{theorem}[Upper Bound]\label{thm: upper_bd}
    \em Fix $\xi \in (0,1)$, $\gamma \in (0,\infty)$. Then the asymptotic regret of ERTS for entropic risk Gaussian MAB bandits satisfies $$\limsup_{n \to \infty} \frac{\ccal R_n(\text{ERTS})}{\log n} \leq \sum_{i \in \sdiff{[K]}{\sett{1}}} R_i \ERgap{i},$$
    where
    $$R_i := \max \left\{\frac{2}{\xi^2\Delta_{\ER}^2(i,\gamma)},\frac{1}{h \paren{\frac{\gamma \sigma_i^2}{\gamma \sigma_i^2-2(1-\xi)\ERgap{i}}}}\right\}.$$
    Furthermore, setting $$\xi_\gamma = 1 - \frac{\gamma \sigma_i^2}{2\ERgap{i}} \paren{1-\frac{1}{\inv {h_+}\paren{ \Delta_{\mathrm{\ER}}^2(i,\gamma)/2}}},$$ yields $$\frac{1}{h \paren{\frac{\gamma \sigma_i^2}{ \gamma \sigma_i^2-2(1-\xi_\gamma)\ERgap{i}}}} \leq \frac{2}{\xi_\gamma^2\Delta_{\ER}^2(i)}$$ and $\xi_\gamma \to 1^-$ as $\gamma \to 0^+$, where $\inv {h_+}(y) = \max \sett{x : h(x) = y}$.
\end{theorem}
\begin{remark}
    \em{The final part of the theorem shows that the upper bound is characterised by the quantity ${2}/{(\xi_\gamma^2\Delta_{\ER}^2(i,\gamma))}$. By continuity, we obtain the regret bound ${2}/{(\Delta_{\ER}^2(i,\gamma))}$. Furthermore, we note that $ \ERgap{i} \to -\mu_i - (-\mu_1) = \mu_1 - \mu_i$ as $\gamma \to 0^+$, and thus the upper bound simplifies to $2/{(\mu_1-\mu_i)}^2$. This agrees with our intuition since $\ER_\gamma(i) = -\mu_i + (\gamma/2)\sigma_i^2 \to -\mu_i$ as $\gamma \to 0^+$, implying that we are in the risk-neutral setting. Thus, the results correspond to those derived for mean-variance bandits \citep{zhu2020thompson} and CVaR bandits \citep{chang2021riskconstrained}.}
\end{remark}
Next, we establish an instance-dependent lower bound for the expected pulls of non-optimal arms under consistent algorithms. Consider a class $\ccal C$ of distributions and define $\ccal S_i = \{\nu'(i) \in \ccal C : \ER(\nu'(i)) < \ER(1)\}$. Define for each non-optimal arm $i$, $$\eta(i,\gamma) = \inf_{\nu'(i) \in \ccal S_i} \{\KL(\nu(i),\nu'(i))\},$$ where $\KL(\PP,\PP')$ denotes the \emph{KL-divergence} between two probability measures $\PP,\PP'$.
\begin{theorem}[Lower Bound]\label{thm: lower_bd}
    \em Let $\pi$ be a policy over the class of distributions $\ccal C$ satisfying $\ccal R_n(\pi) = o(n^a)$ for any $a > 0$. Then for any non-optimal arm $i$, we have $$\liminf_{n \to \infty} \frac{\EE[T_{i,n}]}{\log n} \geq \frac{1}{\eta(i,\gamma)}.$$
    In particular, if $\ccal C= \ccal E_{\ccal N}^K(\sigma_{\max}^2)$, then $$\liminf_{n \to \infty} \frac{\ccal R_n(\pi)}{\log n} \geq \sum_{i \in \sdiff{[K]}{\{1\}}} R_i \ERgap{i}.$$
\end{theorem}
\begin{remark}
    {\em This implies that the asymptotic lower bound for the regret matches its asymptotic upper bound in Theorem~\ref{thm: upper_bd} unconditionally. Hence, for the Gaussian entropic risk MAB problem, ERTS is \emph{asymptotically optimal}.}
\end{remark}
\section{Proof Outlines for Theorem~\ref{thm: upper_bd} and \ref{thm: lower_bd}}\label{sec: pf_outline}
Theorem~\ref{thm: upper_bd}: Denote the sample entropic risk as $\hat{\ER}_{\gamma}(i,t) = -\theta_{i,t} + {\gamma}/{(2\kappa_{i,t})}$. Fix
$\eeps > 0$ and define $E_i(t) := \bb{\hat{\ER}_{\gamma}(i,t)> \ER_{\gamma}(1) + \eeps}$, that is, the event that the Thompson sample mean of arm $i$ is $\eeps$-riskier than a certain threshold or, more precisely, $\eeps$-higher than the optimal arm (which has the lowest entropic risk). Intuitively, event $E_i(t)$ occurs with high probability when the algorithm has explored sufficiently. However, the algorithm does not choose arm $i$ when $E_{i}^c(t)$ occurs with small probability under Thompson sampling, which contributes directly to the regret bound. Therefore, it suffices to bound the number of times $E_{i}^c(t)$ occurs.
\\\\
In order to bound $\E{T_{i,n}}$, we can split $\E{T_{i,n}}$ into two parts using a key lemma by \citet{lattimore_szepesvari_2020} to yield $\E{T_{i,n}} \leq \Lambda_1 + \Lambda_2 +1$, where $\Lambda_1 = \E{\sum^{n-1}_{s=0}\bp{\frac{1}{G_{1,s}}-1}}$ and  $\Lambda_2 = \sum^{n-1}_{s=0}\mathbb{P}\bp{G_{1,s} > \frac{1}{n}}$. It remains to upper bound $\Lambda_1$ and $\Lambda_2$. The techniques to upper bound $\Lambda_1$ are similar to those from \citet[Section 4.6]{zhu2020thompson} and \citet[Section 5]{chang2021riskconstrained}. To upper bound $\Lambda_2$, we split the event $E_i^{c}(t) = \bp{\hat{\ER}_{\gamma}(i,t) \leq \ER_{\gamma}(1) + \eeps}$ into 
\begin{align*}
    \Psi_1(\xi) &= \bb{-\theta_{i,t} +\mu_{i} \leq -\xi(\Delta_{\mathrm{ER}}(i,\gamma) - \eeps)}, \\
    \Psi_2(\xi) &= \bb{\frac{\gamma}{2}\paren{\frac{1}{\kappa_{i,t}} - \sigma_{i}^2} \leq (-1+\xi)(\Delta_{\mathrm{ER}}(i,\gamma) - \eeps)}
\end{align*}
That is, $E_{i}^c(t) \subseteq \Psi_1(\xi) \cup \Psi_2(\xi)$. We then use the union bound which yields $\mathbb{P}(E^c_{i}(t))\leq \PP (\Psi_1(\xi)) + \PP (\Psi_2(\xi))$, which we can upper bound by known concentration bounds. Following the strategy employed by \citet{chang2021riskconstrained}, a judicious selection of the free parameter $\xi \in (0,1)$ allows us to allocate "weights" on the bounds of $\PP(\Psi_1(\xi))$ and $\PP(\Psi_2(\xi))$ which then yield $ {2}/{(\xi^2\Delta_{\ER}^2(i,\gamma))}$ and $\paren{h \paren{\frac{\gamma \sigma_i^2}{\gamma \sigma_i^2-2(1-\xi)\ERgap{i}}}}^{-1}$ without incurring further residual terms. 
\\\\
Theorem~\ref{thm: lower_bd}: The proof of the lower bound follows immediately from \citet[Theorem 4]{kagrecha2020constrained} by replacing the criterion $c_\alpha(\nu'(k)) \leq c_\alpha^*$ by $\ER(\nu'(k)) \leq \ER(1)$. We then particularize the lower bounds therein by decisively setting the distribution of $\nu'(i)$ to have a Gaussian distribution with mean $\mu_i+\sigma_i \sqrt{2/R_i}+\eeps$ and variance $\sigma_i^2$, which then returns the desired lower bound.
\section{Conclusion}\label{sec: conc}
This paper applies Thompson sampling \citep{thompson1933likelihood} to provide the first solution for entropic risk MAB problems which have not been previously considered before to the best of our knowledge. We proposed a new algorithm ERTS to solve this problem and proved that this proposed algorithm is asymptotically optimal for the ER MAB problem. Further work includes analysing Thompson sampling of Gaussian MABs under general risk measures and exploring Thompson sampling's performance for Entropic-Risk sub-Gaussian bandits. We may also potentially design a general framework for proving the efficacy of Thompson sampling over the state-of-the-art L/UCB-based techniques for generalised risk-averse MABs and a wider class of bandits (under reasonable assumptions, such as the crucial properties of the risk-measures, existence of conjugate prior estimates, as well as relevant concentration bounds).
\newpage

\appendix
\section*{Appendix A.}
\label{app:theorem}



\begin{proof}[Proof of Theorem~\ref{thm: upper_bd}]
We first state without proof a crucial lemma from \cite{lattimore_szepesvari_2020} which we will use in our analysis.
\begin{lemma}[\citet{lattimore_szepesvari_2020}]
\label{subthm: lattimore_main}
\em Let $\PP_t(\, \cdot \, ) = \PP(\, \cdot \,  | A_1,X_1,\dots,A_{t-1},X_{t-1})$ be the probability measure conditioned on the history up to time $t-1$ and $G_{is} = \PP_t(E_i^c(t) | T_{i,t} = s)$, where $E_i(t)$ is any specified event for arm $i$ at time $t$. Then $$\EE[T_{i,n}] \leq \sum_{s=0}^{n-1} \EE \parenb{\frac{1}{G_{1s}} - 1} +  \sum_{s=0}^{n-1} \PP \paren{G_{is} > \frac{1}{n}} + 1.$$
\end{lemma}
Denote the sample entropic risk at $\gamma$ by $\widehat{\ER}(i,t) = -\theta_{i,t} + \gamma / (2 \kappa_{i,t})$. Fix $\eeps > 0$, and define
\begin{equation*}
	E_i(t) := \sett{\widehat{\ER}(i,t) > \ER_{\gamma}(1) + \eeps},
\end{equation*}
the event that the Thompson sample entropic risk of arm $i$ is $\eeps$-higher than the optimal arm (which has the lowest entropic risk). Intuitively, event $E_i(t)$ is highly likely to occur when the algorithm has explored sufficiently. However, the algorithm does not choose arm $i$ when $E_i^c(t)$, an event with small probability under Thompson sampling, occurs. By Lemma~\ref{subthm: lattimore_main} and the linearity of expectation, we can divide $\EE[T_{i,n}]$ into two parts as 
\begin{equation}
\label{eqn: key_lemma}
\EE[T_{i,n}] \leq \sum_{s=0}^{n-1} \EE \parenb{\frac{1}{G_{1s}} - 1} +  \sum_{s=0}^{n-1} \PP \paren{G_{is} > \frac{1}{n}} + 1.\end{equation}
By Lemmas~\ref{lem: upper_bound_term_1_risk} and \ref{lem: upper_bound_term_2_risk} by that which follows, we have
\begin{align*}
\sum_{s=1}^{n} \EE \parenb{\frac{1}{G_{1s}} - 1} &\leq \frac{C_1}{\eeps^3} + \frac{C_2}{\eeps^2} + \frac{C_3}{\eeps} + C_4,\ \text{and}\\
\sum_{s=1}^n \PP_t\paren{G_{is} > \frac{1}{n}}
&\leq 1 + \max \sett{\frac{2\log (2n)}{\xi^2\paren{\ERgap{i} - \eeps}^2}, \frac{\log(2n)}{h \paren{\frac{\gamma \sigma^2}{\gamma \sigma_i^2-2(1-\xi)(\ERgap{i} - \varepsilon)}}}} + \frac{C_5}{\eeps^4} + \frac{C_6}{\eeps^2}.
\end{align*}
Plugging the two displays into (\ref{eqn: key_lemma}), we have
\begin{equation}
	\EE[T_{i,n}] \leq 1 + \max \sett{\frac{2 \log (2n)}{\xi^2\paren{\ERgap{i} - \eeps}^2}, \frac{\log(2n)}{h \paren{\frac{\gamma \sigma^2}{\gamma \sigma_i^2-2(1-\xi)(\ERgap{i} - \varepsilon)}}}} + \frac{C_1'}{\eeps^4} + \frac{C_2'}{\eeps^3} + \frac{C_3'}{\eeps^2} + \frac{C_4'}{\eeps} + C_5', \label{eqn: pre_result}
\end{equation}
where $C_1',C_2',C_3',C_4',C_5'$ are constants. Setting $\eeps = {(\log n)}^{-\frac{1}{8}}$ into (\ref{eqn: pre_result}), we get
\begin{align*}
\limsup_{n\to\infty}\frac{\ccal R_{n}{(\text{ERTS})}}{\log n}
&\leq \sum_{i \in \sdiff{[K]}{\sett{1}}} \paren{\max \sett{\frac{2}{\xi^2\Delta_{\mathrm{ER}}^2(i)}, \frac{1}{h \paren{\frac{\gamma \sigma^2}{ \gamma \sigma_i^2-2(1-\xi)\ERgap{i}}}}}}\ERgap{i}.
\end{align*}
\end{proof}

\begin{lemma}
\label{lem: tail_lower_bd}
\em We can lower bound
\begin{equation*}
 	\PP_t\paren{E_1^c(t) \mid T_{1,t} = s, \hat{\mu}_{1,s} = \mu, \hat{\sigma}_{1,s} = \sigma}=\PP_t\paren{\widehat\ER_1 \leq \ER_1 + \varepsilon \mid T_{1,t} = s, \hat{\mu}_{1,s} = \mu, \hat{\sigma}_{1,s} = \sigma}
 \end{equation*}
by
\begin{align}
&\PP_t\paren{\widehat\ER_i \leq \ER_1 + \varepsilon \mid T_{1,t} = s, \hat{\mu}_{1,s} = \mu, \hat{\sigma}_{1,s} = \sigma} \nonumber\\
&\geq \begin{cases}
 	\PP_t\paren{{\theta}_{1,t} - \mu_1\geq -\frac{\eeps}{2}} \cdot \PP_t \paren{\frac{1}{{\kappa_{1,t}}}-\sigma_1 \leq \frac{\eeps}{\gamma}} &\text{if } \mu \leq \mu_1, \sigma \geq \sigma_1,\\
 	\frac{1}{2} \PP_t \paren{\frac{1}{{\kappa_{1,t}}}-\sigma_1 \leq \frac{\eeps}{\gamma}} &\text{if } \mu > \mu_1, \sigma \geq \sigma_1,\\
 	\frac{1}{2} \PP_t\paren{{\theta}_{1,t} - \mu_1\geq -\frac{\eeps}{2}} &\text{if } \mu \leq \mu_1, \sigma < \sigma_1,\\
 	\frac{1}{4} &\text{if } \mu > \mu_1, \sigma < \sigma_1.
 \end{cases}\label{eqn: lower_bd}
\end{align}
\end{lemma}


\begin{proof}
Given $T_{1,t} = s, \hat{\mu}_{1,s} = \mu, \hat{\sigma}_{1,s} = \sigma$, a direct calculation gives us,
\begin{align*}
&\PP_t\paren{\widehat\ER_1 \leq \ER_1 + \varepsilon \mid T_{1,t} = s, \hat{\mu}_{1,s} = \mu, \hat{\sigma}_{1,s} = \sigma} \nonumber\\
&= \PP_t\paren{-\theta_{1,t} + \frac{\gamma}{2\kappa_{1,t}}- \bp{-\mu_1 + (\gamma/2)\sigma_1^2} \leq  \varepsilon\ \Big|\ T_{1,t} = s, \hat{\mu}_{1,s} = \mu, \hat{\sigma}_{1,s} = \sigma}\\
&= \PP_t\paren{-(\theta_{1,t} - \mu_1) + \frac{\gamma}{2}\bp{\frac{1}{\kappa_{1,t}} - \sigma_1^2} \leq  \varepsilon\ \Big|\ T_{1,t} = s, \hat{\mu}_{1,s} = \mu, \hat{\sigma}_{1,s} = \sigma}\\
&\geq \PP_t\paren{-(\theta_{1,t} - \mu_1)  \leq  \varepsilon/2 \mid T_{1,t} = s, \hat{\mu}_{1,s} = \mu, \hat{\sigma}_{1,s} = \sigma}\\ &\quad \; \PP_t\paren{ \frac{\gamma}{2}\bp{\frac{1}{\kappa_{1,t}} - \sigma_1^2} \leq  \varepsilon/2 \ \Big|\ T_{1,t} = s, \hat{\mu}_{1,s} = \mu, \hat{\sigma}_{1,s} = \sigma}\\\
&\geq \begin{cases}
 	\PP_t\paren{{\theta}_{1,t} - \mu_1\geq -\frac{\eeps}{2}} \cdot \PP_t \paren{\frac{1}{{\kappa_{1,t}}}-\sigma_1^2 \leq \frac{\eeps}{\gamma}} &\text{if } \mu \leq \mu_1, \sigma^2 \geq \sigma_1^2,\\
 	\frac{1}{2} \PP_t \paren{\frac{1}{{\kappa_{1,t}}}-\sigma_1^2 \leq \frac{\eeps}{\gamma}} &\text{if } \mu > \mu_1, \sigma^2 \geq \sigma_1^2,\\
 	\frac{1}{2} \PP_t\paren{{\theta}_{1,t} - \mu_1\geq -\frac{\eeps}{2}} &\text{if } \mu \leq \mu_1, \sigma^2 < \sigma_1^2,\\
 	\frac{1}{4} &\text{if } \mu > \mu_1, \sigma^2 < \sigma_1^2.
 \end{cases}\label{eqn: lower_bd}
\end{align*}
Then the lemma holds since $\PP_t(\theta_{1,t} - \mu_1 \geq - \eeps/2) > 1/2$ if $\mu > \mu_1$, and $\PP_t \paren{\frac{1}{\kappa_{1,t}} - \sigma_1^2 \leq \frac{\eeps}{\gamma}} \geq 1/2$ if $\sigma < \sigma_1^2$, by using properties of the median of the Gaussian and Gamma distributions respectively.

\end{proof}

\begin{lemma}[Upper bounding the first term of (\ref{eqn: key_lemma})]
\label{lem: upper_bound_term_1_risk}
\em We have \begin{equation*}
		\sum_{s=1}^{n} \EE \parenb{\frac{1}{G_{1s}} - 1} \leq \frac{C_1}{\eeps^2} + \frac{C_2}{\eeps} + C_3,
	\end{equation*}
	where $C_1,C_2,C_3$.
\end{lemma}
\begin{proof}
The proof follows immediately from Lemma~\ref{lem: tail_lower_bd} and \citet[S-3.3]{zhu2020thompson} by scaling $\eeps >0$.
\end{proof}


\begin{lemma}\label{lem: tail_upper_bd}
\em For $\xi \in (0,1)$, we have
\begin{align*}
    &\PP\paren{\widehat\ER_i \leq \ER_1 + \varepsilon \mid T_{i,t} = s, \hat{\mu}_{i,t} = \mu, \hat{\sigma}_{i,t}^2 = \sigma^2}\\
    &\leq \exp \paren{-\frac{s}{2}\paren{\mu_i - \mu + \xi (\ERgap{i} - \varepsilon)}^2} + \exp \paren{-sh \paren{\frac{\gamma \sigma^2}{ \gamma \sigma_i^2-2(1-\xi)(\ERgap{i} - \varepsilon)}}},
\end{align*}
where $h(x) = \frac{1}{2}(x-1-\log x)$.
\end{lemma}
\begin{proof}
For $\xi \in (0,1)$, we have
\begin{align*}
    &\PP\paren{\widehat\ER_i \leq \ER_1 + \varepsilon \mid T_{i,t} = s, \hat{\mu}_{i,t} = \mu, \hat{\sigma}_{i,t}^2 = \sigma^2}\\
     &= \PP_t\paren{-\theta_{i,t} +\mu_i + \frac{\gamma}{2}\bp{\frac{1}{\kappa_{i,t}} - \sigma_i^2} \leq -\ERgap{i} + \varepsilon \ \Big|\ T_{i,t} = s, \hat{\mu}_{i,s} = \mu, \hat{\sigma}_{i,s} = \sigma}\\
     &\leq \PP_t\paren{-\theta_{i,t} +\mu_i \leq -\xi\bp{\ERgap{i} - \varepsilon}\mid T_{i,t} = s, \hat{\mu}_{i,s} = \mu, \hat{\sigma}_{i,s} = \sigma} +\\ &\quad\; \PP_t\paren{ \frac{\gamma}{2}\bp{\frac{1}{\kappa_{i,t}} - \sigma_i^2} \leq -(1-\xi)\bp{\ERgap{i} - \varepsilon} \ \Big|\ T_{i,t} = s, \hat{\mu}_{i,s} = \mu, \hat{\sigma}_{i,s} = \sigma}\\
     &= \PP_t\paren{\theta_{i,t} - \mu \geq (\mu_i -\mu) + \xi\bp{\ERgap{i} - \varepsilon}\mid T_{i,t} = s, \hat{\mu}_{i,s} = \mu, \hat{\sigma}_{i,s} = \sigma} +\\ &\quad\; \PP_t\paren{ \kappa_{i,t} \geq \frac{\gamma}{  \gamma\sigma_i^2-2(1-\xi)\bp{\ERgap{i} - \varepsilon}} \ \Big|\ T_{i,t} = s, \hat{\mu}_{i,s} = \mu, \hat{\sigma}_{i,s} = \sigma}\\
      &\leq \exp \paren{-\frac{s}{2}\paren{\mu_i - \mu + \xi (\ERgap{i} - \varepsilon)}^2} + \exp \paren{-sh \paren{\frac{\gamma \sigma^2}{\gamma \sigma_i^2-2(1-\xi)(\ERgap{i} - \varepsilon)}}},
\end{align*}
where $h(x) = \frac{1}{2}(x-1-\log x)$.
\\\\
The lemma holds by the Chernoff upper bound for $\PP_t(\theta_{i,t} \geq \cdot )$ and Lemma~\ref{lem: harre} below to upper-bound $\PP_t(\kappa_{i,t} \geq \cdot)$.


\begin{lemma}[\citet{Harremo_s_2017}]
\label{lem: harre}
\em For a Gamma r.v. $X \sim \mathrm{Gamma}(\alpha,\beta)$ with shape $\alpha \geq 2$ and rate $\beta > 0$, we have $$\PP(X \geq x) \leq \exp \paren{-2\alpha h \paren{\frac{\beta x}{\alpha}}},\ x > \frac{\alpha}{\beta},$$ where $h(x) = \frac{1}{2}(x-1-\log x)$.
\end{lemma}
\end{proof}

\begin{lemma}[Upper bounding the second term of (\ref{eqn: key_lemma})]
\label{lem: upper_bound_term_2_risk}
\em We have
\begin{align*}
\sum_{s=1}^n \PP_t\paren{G_{is} > \frac{1}{n}}
&\leq 1 + \max \sett{\frac{2\log (2n)}{\xi^2\paren{\ERgap{i} - \eeps}^2}, \frac{\log(2n)}{h \paren{\frac{\gamma \sigma^2}{ \gamma \sigma_i^2-2(1-\xi)(\ERgap{i} - \varepsilon)}}}} + \frac{C_4}{\eeps^4} + \frac{C_5}{\eeps^2},
\end{align*}
where $C_4,C_5$ are constants.
\end{lemma}
\begin{proof}
Following from Lemma~\ref{lem: tail_upper_bd}, we have the following inclusions:
\begin{align*}
    &\sett{ {\hat \mu_{i,t} + \sqrt{\frac{2\log 2n}{s}} \leq \mu_i + \xi {(\ERgap{i} - \varepsilon)}}}\\
    &\subseteq \sett{\exp \paren{-\frac{s}{2}\paren{\mu_i - \mu + \xi (\ERgap{i} - \varepsilon)}^2} \leq \frac{1}{2n}}
\end{align*}
and 
\begin{align*}
    &\sett{\frac{\gamma \hat \sigma_{i,t}^2}{\gamma \sigma_i^2-2(1-\xi)(\ERgap{i} - \varepsilon)} \leq \inv {h_-}\paren{\frac{\log 2n}{s}}}\\
    &\cup \sett{\frac{\gamma \hat \sigma_{i,t}^2}{ \gamma \sigma_i^2-2(1-\xi)(\ERgap{i} - \varepsilon)} \geq \inv {h_+}\paren{\frac{\log 2n}{s}}}\\
    &\subseteq \sett{\exp \paren{-sh \paren{\frac{\gamma \sigma^2}{\gamma \sigma_i^2-2(1-\xi)(\ERgap{i} - \varepsilon) }}} \leq \frac{1}{2n}},
\end{align*}
where $\inv {h_+}(y) = \max \sett{x : h(x) = y}$ and $\inv {h_-}(y) = \min \sett{x : h(x) = y}$. Hence, for $$s \geq u = \max \sett{\frac{2\log (2n)}{\xi^2\paren{\ERgap{i} - \eeps}^2}, \frac{\log(2n)}{h \paren{\frac{\gamma \sigma^2}{\gamma \sigma_i^2-2(1-\xi)(\ERgap{i} - \varepsilon)}}}},$$
 by replacing $\paren{\mu_1 - \eeps,\frac{\hat\sigma_i^2}{\sigma_1^2+\eeps}}$ with $\paren{\mu_i + \xi \paren{\ERgap{i} - \eeps},\frac{\gamma \hat \sigma_{i,t}^2}{ \gamma \sigma_i^2-2(1-\xi)(\ERgap{i} - \varepsilon)}}$ in \citet[S-3.4]{zhu2020thompson}, we get $$\PP_t\paren{G_{is}>\frac{1}{n}} \leq \exp\paren{-\frac{s\eeps^2}{\sigma_i^2}} + \exp\paren{-(s-1) \frac{\eeps^2}{\sigma_i^4}}.$$
 Summing over $s$,
 \begin{align*}
     &\sum_{s=1}^n \PP_t\paren{G_{is} > \frac{1}{n}}\ \leq u + \sum_{s = \lceil u \rceil}^n \parenb{\exp\paren{-\frac{s\eeps^2}{\sigma_i^2}} + \exp\paren{-(s-1) \frac{\eeps^2}{\sigma_i^4}}}\\
&\leq 1 + \max \sett{\frac{2\log (2n)}{\xi^2\paren{\ERgap{i} - \eeps}^2}, \frac{\log(2n)}{h \paren{\frac{\gamma \sigma^2}{\gamma \sigma_i^2-2(1-\xi)(\ERgap{i} - \varepsilon)}}}} + \frac{C_4}{\eeps^4} + \frac{C_5}{\eeps^2}.
 \end{align*}
Finally, set $$\xi_\gamma = 1 - \frac{\gamma \sigma_i^2}{2\ERgap{i}} \paren{1-\frac{1}{\inv {h_+}\paren{ \Delta_{\mathrm{\ER}}^2(i,\gamma)/2}}} \in (0,1),$$ where $\inv {h_+}(y) = \max \sett{x : h(x) = y}$. By algebra,
\begin{align*}
    &h \paren{\frac{\gamma \sigma_i^2}{\gamma \sigma_i^2-2(1-\xi_\gamma)\ERgap{i}}}\\
    &= h \paren{\frac{\gamma \sigma_i^2}{\gamma \sigma_i^2-2\cdot \frac{\gamma \sigma_i^2}{2\ERgap{i}} \paren{1-\frac{1}{\inv {h_+}\paren{ \Delta_{\mathrm{\ER}}^2(i,\gamma)/2}}} \cdot \ERgap{i}}}\\
    &= h \paren{\frac{\gamma \sigma_i^2}{ \gamma \sigma_i^2 - \gamma \sigma_i^2\paren{1-\frac{1}{\inv {h_+}\paren{ \Delta_{\mathrm{\ER}}^2(i,\gamma)/2}}} }} = h \paren{\frac{\gamma \sigma_i^2}{ \paren{\frac{\gamma \sigma_i^2}{\inv {h_+}\paren{ \Delta_{\mathrm{\ER}}^2(i,\gamma)/2}}}}}\\
    &= h \paren{\inv {h_+}\paren{ \Delta_{\mathrm{\ER}}^2(i,\gamma)/2}} = \Delta_{\mathrm{\ER}}^2(i,\gamma)/2 \geq \xi_\gamma^2\Delta_{ \mathrm{\ER}}^2(i,\gamma)/2
\end{align*}
which implies $$\frac{1}{h \paren{\frac{\gamma \sigma_i^2}{ \gamma \sigma_i^2-2(1-\xi_\gamma)\ERgap{i}}}} \leq \frac{2}{\xi_\gamma^2\Delta_{\ER}^2(i)}.$$ and $\xi_\gamma \to 1^-$ as $\gamma \to 0^+$.
\end{proof}
\begin{proof}[Proof of Theorem~\ref{thm: lower_bd}]
We note that for any arm $i$ with distribution $\nu(i) \sim \mathcal N(\mu_i,\sigma_i^2)$ and $\nu'(i) \sim \mathcal N(\mu_i', {(\sigma_i')}^2)$, the KL-divergence given by $$\KL(\nu(i),\nu'(i)) = \log \frac{\sigma_i'}{\sigma_i} + \frac{\sigma_i^2 + {(\mu_i-\mu_i')}^2}{2{(\sigma_i')}^2} - \frac{1}{2}$$ is well-known. Denote $\mathcal S_i = \sett{\nu'(i) \in {\mathcal E}_{\mathcal N}^K : \ER(\nu'(i)) < \ER(1)}$. Denote $$R_i := \max \left\{\frac{2}{\xi^2\Delta_{\ER}^2(i)},\frac{1}{h \paren{\frac{\gamma \sigma^2}{\gamma \sigma_i^2-2(1-\xi)\ERgap{i}}}}\right\} > 0,$$ and fix $\eeps > 0$ and consider the arm with the distribution $\mathcal{N} \paren{\mu_i+\sigma_i \sqrt{2/R_i}+\eeps, \sigma_i^2}$. Then a direct computation gives
\begin{align*}
    \ER(\nu'(i)) - \ER(\nu(1))
    &= -(\mu_i +\sigma_i \sqrt{2/R_i}+ \eeps) + \frac{\gamma}{2}\sigma_i^2 - \paren{-\mu_i + \frac{\gamma}{2}\sigma_i^2}\\
    &= - (\sigma_i \sqrt{2/R_i}+\eeps) < 0,
\end{align*}
thus $ \ER(\nu'(i)) < \ER(\nu(1))$ and  $\nu'(i) \in \mathcal{S}_i$. Furthermore, 
\begin{align*}
    \KL(\nu(i), \nu'(i)) &= \log \frac{\sigma_i}{\sigma_i} + \frac{\sigma_i^2 + \paren{\mu_i - \paren{\mu_i + \sigma_i \sqrt{2/R_i} + \eeps}}^2}{2\sigma_i^2} - \frac{1}{2}\\
    &= \frac{1}{R_i} + \frac{(2 \sigma_i \sqrt{2/R_i}  + \eeps)\eeps}{2\sigma_i^2}.
\end{align*}
By the definition of $\eta$, $$\eta(i,\gamma) \leq \lim_{\eeps \to 0^+} \parenb{\frac{1}{R_i} + \frac{(2 \sigma_i \sqrt{2/R_i}  + \eeps)\eeps}{2\sigma_i^2}} = \frac{1}{R_i} \Longrightarrow \frac{1}{\eta(i,\gamma)} \geq R_i.$$
Hence, $$\liminf_{n \to \infty} \frac{\mathcal R_n(\pi)}{\log n} = \sum_{i \in \sdiff{[K]}{\sett{1}}} \paren{\liminf_{n \to \infty} \frac{\EE[T_{i,n}]}{\log n}} \ERgap{i} \geq \sum_{i \in \sdiff{[K]}{\sett{1}}} R_i \ERgap{i}.$$
Thus, we have that ERTS is asymptotically optimal unconditionally.
\end{proof}
\vskip 0.2in
\bibliography{ERTS}

\begin{thebibliography}{13}
\providecommand{\natexlab}[1]{#1}
\providecommand{\url}[1]{\texttt{#1}}
\expandafter\ifx\csname urlstyle\endcsname\relax
  \providecommand{\doi}[1]{doi: #1}\else
  \providecommand{\doi}{doi: \begingroup \urlstyle{rm}\Url}\fi

\bibitem[Baudry et~al.(2020)Baudry, Gautron, Kaufmann, and
  Maillard]{baudry2020thompson}
Dorian Baudry, Romain Gautron, Emilie Kaufmann, and Odalric-Ambryn Maillard.
\newblock Thompson sampling for {CVaR} bandits.
\newblock \emph{arXiv preprint arXiv:2012.05754}, 2020.

\bibitem[Chang et~al.(2021)Chang, Zhu, and Tan]{chang2021riskconstrained}
Joel Q.~L. Chang, Qiuyu Zhu, and Vincent Y.~F. Tan.
\newblock Risk-constrained thompson sampling for cvar bandits, 2021.

\bibitem[Galichet et~al.(2013)Galichet, Sebag, and
  Teytaud]{galichet2013exploration}
Nicolas Galichet, Michele Sebag, and Olivier Teytaud.
\newblock Exploration vs exploitation vs safety: Risk-aware multi-armed
  bandits.
\newblock In \emph{Asian Conference on Machine Learning}, pages 245--260, 2013.

\bibitem[Harremo{\"e}s(2016)]{Harremo_s_2017}
Peter Harremo{\"e}s.
\newblock Bounds on tail probabilities for negative binomial distributions.
\newblock \emph{Kybernetika}, 52\penalty0 (6):\penalty0 943--966, 2016.

\bibitem[Howard and Matheson.(1972)]{Howard72}
Ronald~A. Howard and James~E. Matheson.
\newblock Risk-sensitive {Markov} decision processes.
\newblock \emph{Management Science}, 18\penalty0 (7):\penalty0 356--369, 1972.

\bibitem[Kagrecha et~al.(2020)Kagrecha, Nair, and
  Jagannathan]{kagrecha2020constrained}
Anmol Kagrecha, Jayakrishnan Nair, and Krishna Jagannathan.
\newblock Constrained regret minimization for multi-criterion multi-armed
  bandits.
\newblock \emph{arXiv preprint arXiv:2006.09649}, 2020.

\bibitem[Lattimore and Szepesv{\'a}ri(2020)]{lattimore_szepesvari_2020}
Tor Lattimore and Csaba Szepesv{\'a}ri.
\newblock \emph{Bandit algorithms}.
\newblock Cambridge University Press, 2020.

\bibitem[Lee et~al.(2020)Lee, Park, and Shin]{lee2020learning}
Jaeho Lee, Sejun Park, and Jinwoo Shin.
\newblock Learning bounds for risk-sensitive learning.
\newblock In \emph{Advances in Neural Information Processing Systems}, 2020.

\bibitem[Sani et~al.(2012)Sani, Lazaric, and Munos]{sani2013riskaversion}
Amir Sani, Alessandro Lazaric, and R{\'e}mi Munos.
\newblock Risk-aversion in multi-armed bandits.
\newblock In \emph{Advances in Neural Information Processing Systems}, pages
  3275--3283, 2012.

\bibitem[Sun et~al.(2017)Sun, Dey, and Kapoor]{sun2016riskaware}
Wen Sun, Debadeepta Dey, and Ashish Kapoor.
\newblock Risk-aversion in multi-armed bandits.
\newblock In \emph{International Conference on Machine Learning}, pages
  3280--3288, 2017.

\bibitem[Thompson(1933)]{thompson1933likelihood}
William~R Thompson.
\newblock On the likelihood that one unknown probability exceeds another in
  view of the evidence of two samples.
\newblock \emph{Biometrika}, 25\penalty0 (3/4):\penalty0 285--294, 1933.

\bibitem[Vakili and Zhao(2016)]{Vakili_2016}
Sattar Vakili and Qing Zhao.
\newblock Risk-averse multi-armed bandit problems under mean-variance measure.
\newblock \emph{IEEE Journal of Selected Topics in Signal Processing},
  10\penalty0 (6):\penalty0 1093--1111, 2016.

\bibitem[Zhu and Tan(2020)]{zhu2020thompson}
Qiuyu Zhu and Vincent~YF Tan.
\newblock Thompson sampling algorithms for mean-variance bandits.
\newblock In \emph{International Conference on Machine Learning}, pages
  2645--2654, 2020.

\end{thebibliography}

\end{document}